\documentclass{article}
\usepackage{graphicx}
\usepackage{amsmath,amssymb} %
\usepackage{color}
\usepackage{times}
\usepackage{epsfig}
\usepackage{multirow}
\usepackage{natbib}
\usepackage{url}
\usepackage{booktabs}
\usepackage{subcaption}
\usepackage{pgffor}
\usepackage{bm}
\usepackage[utf8]{inputenc} %
\usepackage[T1]{fontenc}    %
\usepackage{hyperref}
\usepackage[margin=2.5cm]{geometry}
\usepackage{amsthm}
\usepackage[affil-it]{authblk}

\def\eqref#1{equation~\ref{#1}}

\def\1{\bm{1}}

\def\vu{{\bm{u}}}
\def\vv{{\bm{v}}}
\def\vw{{\bm{w}}}
\def\vx{{\bm{x}}}
\def\vy{{\bm{y}}}

\def\mD{{\bm{D}}}

\def\mL{{\bm{L}}}

\def\mU{{\bm{U}}}
\def\mV{{\bm{V}}}
\def\mW{{\bm{W}}}

\DeclareMathAlphabet{\mathsfit}{\encodingdefault}{\sfdefault}{m}{sl}
\SetMathAlphabet{\mathsfit}{bold}{\encodingdefault}{\sfdefault}{bx}{n}

\def\gC{{\mathcal{C}}}

\def\gG{{\mathcal{G}}}

\def\gS{{\mathcal{S}}}

\newcommand{\Ls}{\mathcal{L}}
\newcommand{\R}{\mathbb{R}}
\newcommand{\N}{\mathbb{N}}

\newcommand{\normlone}{L^1}
\newcommand{\normltwo}{L^2}

\newcommand\dt[2]{\langle#1,#2\rangle}

\newtheorem{thm}{Theorem}

\newtheorem{lemm}{Lemma}

\newcommand{\chk}{{\centering\checkmark}}

\begin{document}

\title{Exploring Weight Symmetry in Deep Neural Networks}

\author[1]{Xu Shell Hu\thanks{Equal contribution}}
\author[2]{Sergey Zagoruyko$^{*}$}
\author[1]{Nikos Komodakis}
\affil[1]{Universit\'e Paris-Est, \'Ecole des Ponts ParisTech}
\affil[2]{Inria, Paris, France}
\affil[ ]{
          \texttt{sergey.zagoruyko@inria.fr},
          \texttt{\{xu.hu,nikos.komodakis\}@enpc.fr}
         }

\date{}
\maketitle

\vspace{-1cm}
\begin{abstract}
  We propose to impose symmetry in neural network parameters to improve parameter usage and make use of dedicated convolution and matrix multiplication routines. Due to significant reduction in the number of parameters as a result of the symmetry constraints, one would expect a dramatic drop in accuracy. Surprisingly, we show that this is not the case, and, depending on network size, symmetry can have little or no negative effect on network accuracy,
  especially in deep overparameterized networks. We propose several ways to impose local symmetry in recurrent and convolutional neural networks, and show that our symmetry parameterizations satisfy universal approximation property for single hidden layer networks. We extensively evaluate these parameterizations on CIFAR, ImageNet and language modeling datasets, showing significant benefits from the use of symmetry. For instance, our ResNet-101 with channel-wise symmetry has almost 25\% less parameters and only 0.2\% accuracy loss on ImageNet. Code for our experiments is available at~\url{https://github.com/hushell/deep-symmetry}
\end{abstract}

\section{Introduction}

For a long time neural networks had a capacity problem: making them have too many parameters for a limited amount of data would dramatically affect their generalization capabilities. Thus, several regularization techniques were developed, for example, early stopping~\citep{Weigend.Huberman1990Predictingfutureconnectionist} and $\normltwo$ regularization.
The advent of batch normalization~\citep{batch_norm}, skip-connections~\citep{Hochreiter97longshort-term,highway,he2015deep}, and overall architecture search helped to mitigate this problem very recently, so that increasing capacity no longer hurts the accuracy, but even before that it was known that ``unimportant'' connections between neurons can be removed, resulting in a network with significantly fewer parameters and little or no drop in performance.
Optimal brain damage~\citep{optimal_brain_damage} proposed a second order method to prune such connections. Soft weight sharing~\citep{Nowlan:1992:SNN:148167.148169,ullrich_soft_weight_sharing} can be used to very effectively reduce the number of parameters in a trained network. \cite{denil_predict} showed that it is possible to learn part of filters and predict the rest. In fact, parameter sharing is one of the most important features of convolutional neural networks, where sharing is built into convolutional layer, and was thoroughly explored. Recurrent neural networks also heavily rely on sharing parameters over multiple time steps. It is also the key element in training siamese and triplet networks.

More recently, trained on massive amounts of data image recognition neural networks were quickly increasing in the number of parameters, starting from AlexNet~\citep{AlexNet} and VGG~\citep{Simonyan15}.
Network-In-Network~\citep{nin} proposed to stack MLPs which share local receptive fields, removing the need of massive fully-connected layers, and reducing the number of parameters to achieve the same accuracy as AlexNet by several times.
This technique was further adopted by~\cite{GoogLeNet} in their Inception architectures. Downsampling and upsampling can also be seen as a way of parameter sharing in SqueezeNet~\citep{squeezenet}. More recently, Highway~\citep{highway} and later ResNet~\citep{he2015deep} proposed to add skip-connections, which allowed training of very deep networks with improved accuracy. It was later shown by Wide ResNet~\citep{Zagoruyko2016WRN} that the number of parameters was key to their accuracy, and depth was complementary. Either very deep, or very wide residual networks could be trained with massive number of parameters, without suffering from decreased accuracy. After that, parameter sharing exploration continued on ResNet architectures. For example, \cite{Boulch} proposed to share some portion of convolutional filters in each group of residual block, for example every second convolution, achieving relatively small performance loss.

Not all parameter sharing approaches brought both performance improvement and reduction in parameters, and there appears to be some trade-off between parameter sharing and computational efficiency. As an extreme case of parameter sharing with significant performance loss, HyperNetworks~\citep{hypernetworks}, proposed to have another network to generate filters. Among methods improving computational efficiency, a good example is MobileNet~\citep{mobilenets}, which suggested to reparameterize each \(3\times3\) convolution as a pair of depth-wise convolution and \(1\times1\) convolution, can also be seen as a way of sharing parameters. Their work was later extended to grouped \(1\times1\) convolution by ShuffleNet~\citep{shufflenet}, which further reduced the complexity of residual block.

Another approach to reducing number of parameters is post processing, when a trained networks is modified, either with or without additional fine tuning. A number of low rank approximation by tensor decomposition approaches were proposed, such as works of \cite{jaderberg_2014} and \cite{vlebedev2015}. For example, \cite{denton_svd} used low-rank decomposition of the weight matrices to reduce the effective number of parameters in the network. Such approaches tend to be less applicable as more parameter-effective architectures are being invented.
A very effective way of reducing parameters in a trained network is deep compression proposed by~\cite{han2015deep_compression}, with a reduction of parameters of dozens of times, achieved by pruning, weight sharing, quantization and compression. Weight sharing locations are determined by weight values in a trained network, similar in value neurons share the same weight. Common approach today is to train a massive network, and reduce it later. This can even improve results over single-time trained network~\citep{song_han_dsd}.

Recurrent neural networks are known to be significantly overparameterized as well.
For example, \cite{Kim2016SequenceLevelKD} show that it is possible to reduce neural machine translation  model size by $13\times$ with an insignificant BLEU score drop by doing teacher-student knowledge distillation. Several works~\citep{Merity2016PointerSM,Inan2016TyingWV,Kim2016CharacterAwareNL} focus on reducing number of parameters in language modeling tasks.

All of the above suggest that architectures and methods used for training deep neural networks are suboptimal, and could be significantly improved by organizing parameter sharing from the start. It would be interesting to learn it, and a few attempts were made, e.g.\ in DCT space with hashing~\citep{compressing_chen} and FFT space~\citep{mathieu_fft}. Also, automatic architecture search approaches, such as~\citep{zoph17}, do not handle sharing.%

Despite so much work on reducing neural network parameters, surprisingly, very little attention has been given to weight symmetry. To the best of our knowledge, it has not been used in the context of parameter sharing so far. We propose a number of possible approaches to enforce symmetry on weights \textit{from scratch}, i.e.\ not using a pretrained network. Imposed correctly, such symmetry can significantly reduce the number of parameters and computational complexity by sacrificing little or no accuracy, both at train and test time. Post-processing methods such as deep compression could be applied later.
Also, our parameterizations are simple to implement in any modern framework with automatic differentiation. Besides, specialized routines for symmetric matrix multiplication and convolution could be used in both training and testing~\citep{hpc_blas,rajib_symv,Igual09BLAS3GPU}.

We believe that our findings are surprising and uncover interesting properties of deep neural networks, which could led to further advances in understanding and efficiency. Our contributions are summarized below.

\begin{itemize}
  \item We propose an effective use of symmetry for model compression for the first time;
  \item Explore various ways of imposing symmetry constraints;
  \item Experimentally show that symmetry can successfully be imposed in various architectures and datasets, with no or little loss in accuracy;
  \item We show that our symmetric parameterization is generic and can be applied to both image classification and natural language processing tasks;
  \item We provide theoretical proof that neural networks with symmetric parameterization are universal approximators.
\end{itemize}

\section{Symmetric reparameterizations}

\newcommand{\WW}{\mW}

In this section we introduce several ideas to impose axial symmetry on the weights of convolutional / fully-connected layers, which prunes out a large fraction of redundant parameters and gains computational speed-up for both training and testing.
Throughout this section we consider all operations applied to matrices, which can be easily extended to multidimensional tensors by block symmetry.

\subsection{Motivation}

Real symmetric matrices have only real eigenvalues, while rotation matrices whose rotation angles are positive have at least one complex eigenvalue. Thus, restricting the weight matrix to be symmetric for some layers is equivalent to forcing these layers to learn non-rotational transformations. In a deep neural network, this restriction, which can be considered as an inductive bias, should not hurt the overall performance, since different layers are encouraged to focus on particular transformations, and then a good representation of the input is built up by compositing all of them.

On the other hand, introducing symmetric weights for network compression have multiple benefits:
\begin{itemize}
  \item Symmetric matrix multiplications could deliver potential speedup over generic BLAS routines.
  \item Theoretical guarantee in terms of universal approximation (see Section~\ref{sec:theo}).
  \item Training from scratch: it avoids additional fine tuning compared with post-processing methods~\citep{jaderberg_2014,vlebedev2015,han2015deep_compression}.
  \item Easy to combine with other compression methods (e.g. with ShaResNet \citep{Boulch}).
\end{itemize}

\subsection{Soft constraints}

We first try to see if it is possible to enforce symmetry constraints in a soft manner during training, which has the advantage that training procedure remains very similar to standard supervised training.
We start by adding a penalty term to the training loss, which leads to a soft constraint on $\WW$:
\begin{align}
  \Ls(\WW) + \rho \big\| {\rm{vec}}(\WW) - {\rm{vec}}({\WW}^\top) \big\|_p,
  \label{eq:soft}
\end{align}
where $\Ls(\WW)$ is a task-specific loss with respect to $\WW$; $\rho$ is a hyper-parameter to control the slackness of the constraints $\WW = {\WW}^\top$; ${\rm{vec}}(\mW)$ is an operator that vectorizes the matrix into a column vector. For the norm of the penalty term, we use either $p=2$ or $p=1$. However, $\normlone$-norm turns out to be slightly more effective given that it promotes sparsity, which would possibly result in a larger number of exactly symmetric weights.

Due to the soft constraints, the resulting matrix is not guaranteed to be symmetric and so at test time we use only the upper triangular part.

\subsection{Hard constraints}
\label{sec:hard_constr}

Alternatively, we can make use of a specific parameterization by a linear operator $T \colon \mW \mapsto \hat{\mW}$ for some layer in the neural network that explicitly encode symmetry, where $\mW$ is the actual weight and $\hat{\mW}$ is the constructed weight for that layer. We are interested in the case where $\mW$ has fewer elements than $\hat{\mW}$, which enables a reduction in weights while keeping exactly the same network architecture as if it is fully-parameterized. Note that we choose $T$ to be nonparametric, thus the only additional burden of training is to forward and backward propagate through $T$; while in testing, since $\hat{\mW}$ is symmetric, we only need to store the upper triangular
of the learned $\hat{\mW}$, which saves almost half of the space.

To this end, we propose different instantiations of $T({\mW})$ in terms of different ways to construct symmetric matrices.

\paragraph{Triangular parameterization}
One of the simplest way to impose hard symmetry constraint is to define the linear operator $T$ as a sum of the upper triangular matrix of $\mW \in \R^{N \times N}$, its transpose $\mW^\top$ and the diagonal vector \(\vv \in \R^N\):
\begin{align}
  \hat{\mW} = T({\rm{triu}}(\mW), \vv) := {\rm{diag}}(\vv) + {\rm{triu}}(\mW) + {\rm{triu}}(\mW)^\top,
\end{align}
where ${\rm{triu}}(\mW)$ returns the upper triangular part of the matrix $\mW$.
Note that we use the full matrix $\mW$ in the above expression just to simplify the notation. In fact, the triangular parameterization stores only the upper triangular of $\mW$ and the main diagonal $\vv$, while elements of the lower triangular will never be touched. Thus, the number of parameters in this parameterization is $\frac{1}{2} N(N + 1)$, reduced by almost 2 times compared to full parameterization. %
Note that (${\rm triu}(\mW)$, $\vv$) should be initialized from the same distribution as if $\mW$ is learned directly. Due to strong sharing,
the gradient with respect to $\mW$ will be twice higher in magnitude, so the learning rate needs to be adjusted accordingly.

\paragraph{Average parameterization}
Let us also consider a redundant, but more straightforward formulation of the reparameterization $T$, in which we keep \(N\times N\) matrix $\mW$ as the actual weight, and define $T$ as a sum of $\mW$ and its transpose \(\mW^\top\) divided by two:
\begin{align}
  \hat{\mW} = T(\mW) := \frac{1}{2}(\mW + \mW^\top).
\end{align}
Although this is a redundant parameterization to obtain a symmetric matrix, we wanted to explore its performance given that it is known that overparameterized networks are often easier to optimize / train (e.g., directly training compact networks is much more challenging compared to first training overparameterized ones and then properly pruning their parameters).
In this case, the learning rate remains the same since the gradient has been scaled by definition.
As in triangular parameterization, $\mW$ is initialized from the same distribution as basic non-symmetric parameterization. The number of parameters is the same as non-symmetric parameterization at training time, but \(\frac{1}{2}N(N+1)\) at testing time due to the fact that $\hat{\mW}$ is symmetric.

\paragraph{Eigen parameterization}
In addition, we consider a more generic {eigen parameterization}, inspired by the fact that any symmetric matrix has an eigen decomposition. Given a matrix \(\mV \in \R^{N\times R}\) and a vector \(\mathbf{\lambda} \in \R^R \) as actual weights, $T$ is defined by
\begin{align}
  \hat{\mW} = T(\mV, \mathbf{\lambda}) := \mV {\rm{diag}}(\mathbf{\lambda}) \mV^\top.
\end{align}
Ideally, $\mV$ has to be an orthogonal matrix, but we relax this constraint due to a heavy computation of performing projected stochastic gradient descent. We still initialize $\mV$ and $\mathbf{\lambda}$ from an eigen decomposition of the initial full weight matrix. The number of parameters in such relaxed parameterization is \(N(R+1)\) at training time, and \(\frac{1}{2}N(N+1)\) at testing time. %
We empirically choose $R = N/2$ as it yields similar performance as the case of $R=N$.

Note that a similar approach was suggested by~\cite{denil_predict}, where $\hat{\mW}$ is parameterized by matrices \(\mU\) and \(\mV\) with columns of \(\mU\) forming a dictionary of basis functions.

\paragraph{LDL parameterization}
There is a close relationship between the eigen decomposition of a matrix and its LDL decomposition. Recall that the LDL decomposition factorizes a matrix as a product of an unit lower triangular matrix $\mL$ (meaning that all elements on the diagonal are $1$'s), a diagonal matrix $\mD$ and the transpose of $\mL$. The advantage of using LDL decomposition over eigen decomposition is that it is much easier to maintain a valid decomposition of $\hat{\mW}$ during training.
We thus also consider a {LDL parameterization} with additional assumptions\footnote{If $\hat{\mW}$ is symmetric and it factorizes as $\hat{\mW} = \mL\mD\mU = \mU^\top \mD \mL^\top$, then by uniqueness, it follows that $\hat{\mW} = \mL\mD\mL^\top$. However, $\hat{\mW}$ has a LU decomposition if $\hat{\mW}$ satisfies a particular rank condition studied by \cite{okunev2005necessary}. Thus, we in fact assumes $\hat{\mW}$ is better conditioned.} on $\hat{\mW}$. Specifically, the reparameterization $T$ with actual weights $\mL$ and $\mD$ is given by
\begin{align}
  \hat{\mW} = T(\mL,\mD) := \mL \mD \mL^\top,
\end{align}
where $\mL$, $\mD$ are restricted to be unit lower triangular matrix and diagonal matrix respectively.

\paragraph{N-way symmetry parameterization}
Inspired by triangular parameterization, which can be viewed as an axial symmetry about the main diagonal, we consider a more general N-way parameterization with respect to multiple axes of symmetry, where N denotes the number of repeated parts. Thus, previously introduced parameterizations are 2-way symmetries.

\begin{figure}[t]
  \centering
  \foreach \figname in {orig,block_edited,quar_edited,eight_edited}
  {\subcaptionbox{}[0.19\columnwidth]{\includegraphics[width=0.17\columnwidth]{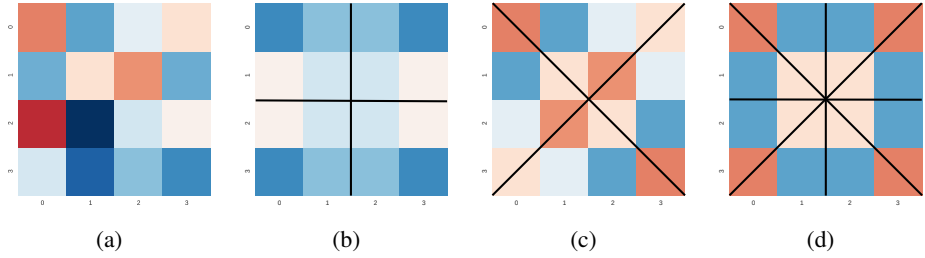}}}
    \caption{N-way parameterizations.
    (a) Original $4 \times 4$ weight matrix.
    (b) 4-way blocking: $\mV$ is the bottom-right block; $\hat{\mW} = \text{reflect}_{-}(\text{reflect}_{|}(\mV))$.
    (c) 4-way triangulizing: $\mV$ is the top triangle; $\hat{\mW} = \text{reflect}_{/}(\text{reflect}_{\backslash}(\mV))$.
    (d) 8-way triangulizing: $\mV$ is the top-left triangle; $\hat{\mW} = \text{reflect}_{/}(\text{reflect}_{\backslash}( \text{reflect}_{|}(\mV)))$.}
  \label{fig:n-way}
\end{figure}

Given $\mV$ as the actual weight, which is considered in general to be smaller than the required weight matrix $\mW$, we create a symmetrized version of $\mW$ by a composition of linear transformations
\begin{align}
  \hat{\mW} = T_1 \circ T_2 \circ \ldots \circ T_k( \mV ),
\end{align}
where $T_j(\cdot)$ is a basic linear operator which does one of the following things: translation, reflection, rotation, tessellation etc. In fact, the form of N-way symmetry is quite flexible. We consider only the cases where symmetry leads to efficient computations in testing. In this work, we focus on two N-way parameterizations: \emph{blocking} and \emph{triangulizing}. The details are listed as follows.
\begin{itemize}
\item Blocking: Given $\mV$ as a $\frac{M}{\sqrt{N}} \times \frac{M}{\sqrt{N}}$ matrix, $N$-way blocking can be obtained by a series of reflections (denoted by $\text{reflect}_\text{axis}(\cdot)$) to include mirrors.
\item Triangulizing: Given $\mV$ as an isosceles right triangle with area $\frac{M^2}{N}$, $\hat{\mW}$ is obtained by a series of reflections about different axes.
\end{itemize}
We demonstrate several examples of N-way symmetries in Figure~\ref{fig:n-way}.

\subsection{Combining with other methods}\label{methods:connection}

It is not surprising that hard-constrained symmetry parameterization can be viewed as a special weight sharing method. Nevertheless, symmetry parameterization is capable of taking advantage of special matrix computation routines to speed up both training and testing, which is not the case for unstructured weight sharing methods such as duplicating randomly picked elements to form $\hat{\mW}$.

Symmetry parameterization can also be complementary to other parameter reduction or weight sharing methods. For example, we force weight sharing not only between residual blocks within the same stage \citep{Boulch} but also within the same residual block using symmetry parameterizations. We show in Table~\ref{table:cifar_architectures} that triangular parameterization can be combined with ResNeXt~\citep{Xie2016} and MobileNet~\citep{mobilenets},
which have already been designed to take advantage of weight sharing.
Besides, post-processing methods \citep{han2015deep_compression} can be applied on the trained symmetric weights with fine tuning,
thus further reducing the number of parameters.
We show experiments on combing channel-wise symmetry
with other parameter reduction methods (e.g. ShaResNet \citep{Boulch}) in Section~\ref{sec:ShaResNet}. %

\subsection{Theoretical analysis}
\label{sec:theo}

As a motivation of applying symmetry from a theoretical point of view, we show that a symmetric feedforward neural network with one fully-connected hidden layer admits the universal approximation property. The analysis is based on the result of \cite{cybenko89}.
We briefly explain here the high level ideas behind the proof of the universal approximation property. %

The building block of a feedforward neural network is the sigmoid thresholded linear function of the form
\begin{align}
  \sigma_\vw(\vx) := \sigma(\dt{\vw}{\vx}),
\end{align}
where $\sigma(t)$ is a sigmoidal function such that $\sigma(t) \rightarrow 1$ as $t \rightarrow +\infty$ and $\sigma(t) \rightarrow 0$ as $t \rightarrow -\infty$. Then, a feedforward neural network $g_{\alpha,\mW}$ with a single hidden layer is a finite linear combination of $\{\sigma_{\vw_i}\}_i$ with $\vw_i$ the $i$-th row of $\mW$:
\begin{align}
  g^m_{\alpha,\mW}(\vx) := \sum_{i=1}^m \alpha_i \sigma_{\vw_i}(\vx) = \dt{\alpha}{\sigma(\mW \vx)}.
\end{align}
Define the sets $\Sigma := \{ \sigma_\vw \colon \vw \in \R^n \}$ and $\gG := \{ g^m_{\alpha,\mW} \colon m \in \N, \mW \in \R^{m\times n}, \alpha \in \R^m \}$. It is easy to see that $\gG = {\rm span}(\Sigma)$. The main result of \cite{cybenko89} shows that, on the domain $[0,1]^n$, $\gG$ is a dense subset of the set of continuous functions (denoted by $\gC([0,1]^n)$).

Now, consider the set $\gS := \{ g^m_{\alpha,\mW} \in \gG \colon m = n, \mW = \mW^\intercal \}$. We show that $\gS$ is still a dense subset of $\gC([0,1]^n)$. Equivalently, we formally state the result as follows.

\begin{lemm}[\cite{cybenko89}]
  \label{thm:1}
  The function $\sigma_\vw(\vx)$ satisfies the property that if
  $
  \int_{[0,1]^n} \sigma_{\vw}(\vx) d\mu(\vx) = 0\,
  $
  holds for all $\vw \in \R^n$, then $\mu = 0$.
\end{lemm}

\begin{lemm}%
  \label{thm:2}
  For any subset $\gS$ of $\gC([0,1]^n)$, such that $\Sigma \subseteq \gS$, it follows that $\gS$ is a dense subset of $\gC([0,1]^n)$.
\end{lemm}
\begin{proof}
  Assume to the contrary that $\gS$ is not dense, which means $\overline{\gS}$ is a closed proper subspace of $\gC([0,1]^n)$. By Hahn-Banach theorem, let $f \in \gC([0,1]^n) \setminus \overline{\gS}$, there exists a bounded linear functional $L \in \gC^*([0,1]^n)$ such that $L(g) = 0 \,\,\forall g \in \overline{\gS}$ and $L(f) = \text{dist}(f,g) > 0$.

  By Riesz representation theorem, for all $L \in \gC^*([0,1]^n)$, there is a unique $\mu \in \gC([0,1]^n)$, such that $L(f) = \dt{f}{\mu}$ for all $f \in \gC([0,1]^n)$ and $\| L \| = \| \mu \|$. Note that $\forall \vw \colon \sigma_\vw \in \overline{\gS}$ by definition. Thus, $L(\sigma_\vw) = \dt{\sigma_\vw}{\mu} = 0$ for all $\vw$, which implies $\mu = 0$ by Lemma~\ref{thm:1}, and further implies that $\| L \| = 0$: a contradiction. Hence, $\gS$ is a dense subset of $\gC([0,1]^n)$.
\end{proof}

\begin{thm}
  Given any $f \in \gC([0,1]^n)$ and $\epsilon > 0$, there is a function $g^n_{\alpha,\mW} \in \gS$ satisfying
  $
  | f(\vx) - g^n_{\alpha,\mW}(\vx) | < \epsilon \,\,\text{for all}\,\, \vx \in [0,1]^n.
  $
\end{thm}
\begin{proof}
  For any symmetric matrix $\mW$, we have
  $$
    \mW = \mU \Lambda \mU^{-1},
  $$
  where $\mU$ is an orthonormal matrix (let $\vu_i$ be its $i$th row) and $\Lambda$ is a diagonal matrix with diagonal elements being eigenvalues. Note that, for all $\vx \in \R^n$, there exists $\vy \in \R^n$, such that $\vx = \mU \vy$. Hence, $\mW \vx = \mU \Lambda \vy$. Then, $\gS$ can be rewritten as
  $$
  \gS = \{ g^m_{\alpha,\mU \Lambda} \in \gG \colon m = n,  \dt{\vu_i}{\vu_j} = 0 \,\,\forall i \neq j,  \dt{\mathbf{1}}{\vu_i} = 1 \,\,\forall i \}.
  $$
  Note that for a $n\times n$ orthogonal real matrix, the degree of freedom is $n^2 - \frac{n(n-1)}{2} \geq n$ given the orthogonality constraints. Without loss of generality, we assume that $\vu_i \Lambda \in \R^n$ is the row of $\mU \Lambda$ with the full degree of freedom. Then, for any function $\sigma_{\vw} \in \Sigma$, there exists $g^n_{\alpha,\mU \Lambda} \in \gG$ satisfying $g^n_{\alpha,\mU \Lambda} = \sigma_{\vw}$ by setting $\vw = u_i \Lambda$ and choosing $\alpha_i = 1$, $\alpha_j = 0 \,\,\forall j \neq
  i$. This implies $\Sigma \subseteq \gS \subseteq \gG$, and the result follows immediately from Lemma~\ref{thm:2}.
\end{proof}

\section{Implementations of block symmetry}

The proposed symmetry parameterization is a generic parameter reduction method, which can be easily adapted to various network architectures.
For example, convolutional / linear layers in feedforward networks (such as VGG~\citep{Simonyan15}, ResNeXt~\citep{Xie2016}, MobileNet~\citep{mobilenets} etc.) and in recurrent networks (such as LSTM~\citep{Hochreiter97longshort-term}, GRU~\citep{KyunghyunChoGRU} etc.) can be symmetrized. In the case of grouped convolution in ResNeXt, block symmetry can be imposed on each group kernel: for instance, suppose that filters are of shape \(g \times N \times N \times k \times k\), where \(g\) is the number of groups, we can impose symmetry on dimensions of \(N\times N\). However, symmetry is not directly applicable to DenseNet \citep{huang2016densely}, as the number of filters grows with every layer. We conduct some experiments along this direction in Section~\ref{sec:exp_cifar}.

\subsection{Imposing symmetry in convolutional neural networks}
We denote by $(\WW, \tilde{\WW})$ the whole set of parameters of the convolutional neural network, where $\WW$ is the subset of parameters to be symmetrized, and $\tilde{\WW}$ is the subset of free parameters. To be more specific, we assume there are totally $L$ convolutional layers being reparameterized to equip symmetric parameters. That is, $\WW := \{ \WW^l \}_{l=1}^L$ with $\WW^l$ satisfying certain symmetric properties:
\(\WW^l \in \R^{N_o \times N_i \times K_h \times K_w}\) is constructed so that the number of input channels is equal to the number of outputs channels (i.e. \(N_i = N_o = N\)) and the spatial domain is a square (i.e. \(K_h = K_w = K\)). We propose to impose symmetry on slices of $\WW^l$, namely, on the slice $\WW^l_i$, which is a square matrix.

Depending on which direction to slice the tensor, we have \emph{channel-wise symmetry} and \emph{spatial symmetry}. In general, we can write $\WW^l := \{ \WW^l_i \}_{i \in I}$. For channel-wise symmetry, $\WW^l_i$ is a $N \times N$ symmetric matrix, and $I := \big\{(k_h, k_w) \mid k_h, k_w \in \{1,\ldots,K\} \big\}$; For spatial symmetry, $\WW^l_i$ is a $K \times K$ symmetric matrix, and $I := \big\{(k_i, k_o) \mid k_i, k_o \in \{1,\ldots,N\} \big\}$.
Since these two symmetries are not exclusive, we can indeed impose both at the same time.

In theory, both channel-wise and spatial symmetries will reduce the freedom of layers, and their ability to approximate functions. Enforcing them to a shallow network can be problematic, since it may significantly reduce the expressive power of the network. Deep highway and residual networks, on the other hand, are more robust to symmetric weights, since they have many layers that are capable to make relatively small changes and iteratively improve the representation of the input \citep{unrolled_iterative}.
In addition, spatial symmetry can be further motivated by the success of scattering networks \citep{bruna2013invariant}, whose filters are fixed as wavelets and constructed to enjoy certain symmetric properties.

We show in experiments that the proposed symmetries can be applied to several modern deep neural networks without suffering a significant drop in both training and testing accuracy.

\subsection{Imposing symmetry in recurrent neural network}

We chose perhaps the most popular RNN variant, LSTM, which is known to be overparameterized, to experiment with symmetry.
In a LSTM cell, the weight matrices between the hidden unit and gates (input, forget, output) as well as the weight matrix between the hidden unit and itself are square, so it is immediately valid to apply aforementioned symmetry parameterizations. This reduces about $25\%$ parameters from the standard LSTM. We show in Section~\ref{sec:language} that the symmetrized LSTM works as well as the standard LSTM in language modeling.

\section{Experiments}

This section is composed as follows. We start with CIFAR experiments, where we first test various symmetry parameterizations on wide residual networks (WRN) by~\cite{Zagoruyko2016WRN}. After determining which parameterizations work best, we determine in which layers symmetry can be applied. We then test it with WRN of different widths and depths to determine the best configuration in terms of parameter reduction, computational complexity and simplicity. We also apply the proposed symmetrization to other architectures, and show that the conclusions drawn from CIFAR are able to transfer to larger datasets (ImageNet-1K dataset). Finally, we apply symmetrization to language modeling tasks.

We emphasize that our goal here is not to show state-of-the-art accuracy, but to show that very simple symmetry constraints can be used to significantly reduce the number of parameters in various network architectures.

There are two common ResNet variants are considered as baselines in the following experiments: \emph{basic} blocks and \emph{bottleneck} blocks~\citep{he2015deep}. Basic blocks have two \(3\times3\) convolutional layers and a parallel residual connection; bottleneck block is a combination of \(1\times1\), \(3\times3\) and \(1\times1\) convolutional layers. Similar to WRN, we refer to WRN-\(n\)-\(k\)-blocktype as the network of depth \(n\), width \(k\) (number of channels multiplier), and blocktype meaning either basic or bottleneck.

Code for all our experiments is available at~\url{https://github.com/hushell/deep-symmetry}.

\subsection{CIFAR experiments}
\label{sec:exp_cifar}

The results of various experiments on CIFAR-10/100 datasets are presented below.

\subsubsection{Symmetry parameterizations}
We first compare various symmetry parameterizations on CIFAR-10 with WRN-16-1-bottleneck, which is a relatively small network enabling us to perform quick experiments. The median validation errors over 5 runs are reported in Table~\ref{table:cifar_parameterizations} as well as a comparison in terms of the number of parameters needed in training and testing. %

\newcommand{\arch}[1]{#1}
\begin{table*}[ht]
    \centering\small
    \begin{tabular}{lcccc}
      \toprule
      \multirow{2}{*}{symmetry parameterization} & \multicolumn{2}{c}{\#parameters} & \multirow{2}{*}{CIFAR-10} \\
       & train & test & \\
      \midrule
      \arch{baseline (non-symmetric)}             & 0.219M & 0.219M & 8.49 \\
      \arch{$\normlone$ soft constraints}                     & 0.219M & 0.172M & 8.61 \\
      \arch{channelwise-triangular}               & 0.172M & 0.172M & 8.84 \\
      \arch{channelwise-average}                  & 0.219M & 0.172M & 8.83 \\
      \arch{channelwise-eigen}                    & 0.173M & 0.173M & 10.23 \\
      \arch{channelwise-LDL}                      & 0.172M & 0.172M & 9.15 \\
      \arch{spatial-average}                      & 0.219M & 0.187M & 9.70 \\
      \arch{spatial\&channelwise-average}         & 0.219M & 0.156M & 10.20 \\
      \bottomrule
    \end{tabular}
    \caption{Various parameterizations on CIFAR-10 with WRN-16-1-bottleneck. We show median error over 5 runs and the numbers of parameters used in training and testing.}
    \label{table:cifar_parameterizations}
\end{table*}

All symmetry parameterizations have certain drop in performance comparing to the baseline. We observe that channel-wise triangular / average parameterizations have the lowest drop. Eigen parameterization does not work well in this experiment, which is possibly a consequence of $\mV$ is not forced to be an orthogonal matrix. On the other hand, LDL parameterization as an alternative attains a better result. We also test soft channel-wise $\normlone$-norm (see eq.~\ref{eq:soft}), where the
number of parameters remains the same in training, and reduced at test time by using upper triangular weights only. The slackness is controlled by the coefficient \(\rho\). For a large $\rho$, the soft-constrained symmetrization is slightly better than hard-constrained symmetrizations.

Spatial symmetry parameterizations do not work as well as channel-wise symmetry. This is expected since the size of spatial dimensions is much smaller compared with channel dimensions (i.e. $N_i = N_o > k_h = k_w$). Thus, the constraints imposed on spatial dimensions are much harsher making the learning much more difficult.

Based on the aforementioned analysis, we choose triangular parameterization (in terms of validation accuracy, parameter reduction and simplicity both at train and test time) as the main method to conduct all experiments further in this section.

\subsubsection{Wide Residual Networks with symmetry}
In this section we compare symmetrized WRN to its wider and thinner counterparts, and discuss the choice of network architecture. We use the terms \textit{conv0}, \textit{conv1} and \textit{conv2} to refer to the first, the second and the third convolutional layers respectively (basic blocks have only \textit{conv0} and \textit{conv1}). For our experiments we choose the bottleneck and constrain the mid-bottleneck $3\times3$ \textit{conv1} convolution to be symmetric, keeping $1\times1$ \textit{conv0} and \textit{conv2} unconstrained. This choice is because the approximating power of the residual block is least reduced, as real eigenvalues of \textit{conv1} are rotated by the surrounding convolutions, resulting in overall rich parameterization.

We present results for WRN-40-1-bottleneck and WRN-40-2-bottleneck trained on CIFAR in Table~\ref{table:cifar_resnet40}. We also train thinner networks with the same number of parameters to compare with their symmetric variants. On both datasets symmetric parameterization compares favorably to both wider and thinner non-symmetric counterparts in terms of accuracy and number of parameters.

\begin{table*}[ht]
  \centering\small
  \begin{tabular}{lcccccc}
    \toprule
    base network & symmetry location & width & \#params & CIFAR-10 & CIFAR-100 \\
    \midrule
    \multirow{3}{*}{WRN-40-1-bottleneck}
    & none    & 1     & 0.59M & 6.12 & 26.86 \\
    & none    & 0.875 & 0.45M & 6.29 & 28.36 \\
    & conv1   & 1     & 0.45M & 6.24 & 27.66 \\
    \midrule
    \multirow{3}{*}{WRN-40-2-bottleneck}
    & none & 2.0 & 2.34M & 4.95 & 22.51 \\
    & none & 1.75 & 1.79M & 5.12 & 23.18 \\
    & conv1 & 2.0 & 1.76M & 4.96 & 22.98 \\
    \bottomrule
  \end{tabular}
  \caption{CIFAR test error (median of 5 runs) of triangular channel-wise parameterization on WRN-40 with bottleneck layers. \textit{conv0} and \textit{conv1} refer to the first and the second convolutional layers respectively.}
  \label{table:cifar_resnet40}
\end{table*}

We further illustrate this in Fig.~\ref{fig:cifar_basic_and_bottleneck}, where we show training and validation accuracy of WRN with respect to various depths and widths. WRN has a lower accuracy in shallower networks when the training accuracy does not reach 100\%, that is, the network struggles to fit into training data. In such cases symmetry constraints damage both training and validation accuracy significantly. Here we also notice that symmetry constraints cause much smaller accuracy drop in networks which are able to fit into training data perfectly, having almost 100\% training accuracy. We hereafter refer to such networks as \textit{overparameterized}, and we should note that overparameterization should not be confused with overfitting, that is, overparameterized network do not suffer from poor generalization.

\begin{figure*}[ht]
  \subcaptionbox{WRN of various depth and width with bottleneck blocks and triangular symmetry. Dash (solid) lines denote train (val) accuracy respectively (medians over 5 runs).\label{fig:cifar_basic_and_bottleneck}}[0.48\textwidth]{
  \includegraphics[width=0.5\textwidth]{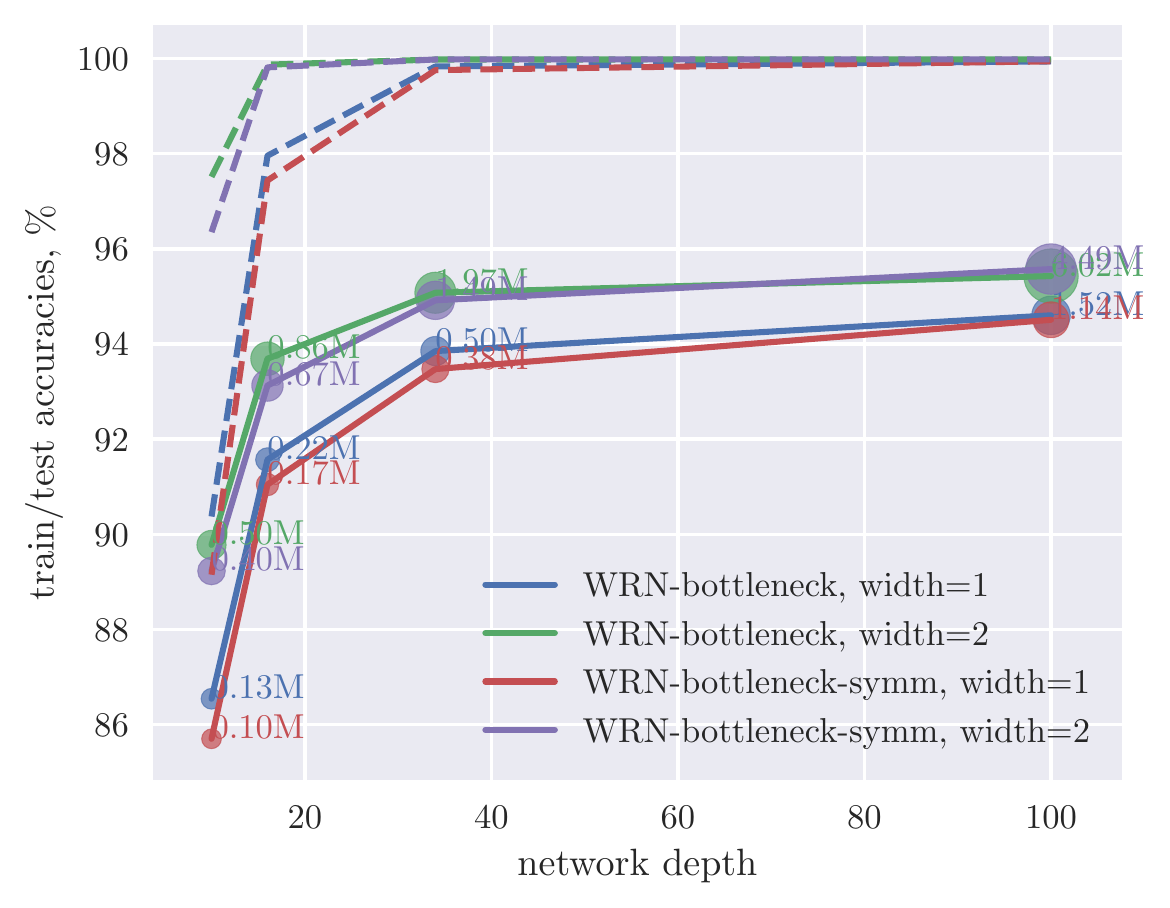}
  }
  \hfill
  \subcaptionbox{Convergence curves of top-1 (top lines) and top-5 (bottom lines) validation errors of ResNet-50 and its triangularly symmetrized variant on ImageNet. \label{fig:imagenet_resnet50}}[0.48\textwidth]{
  \includegraphics[width=0.485\textwidth]{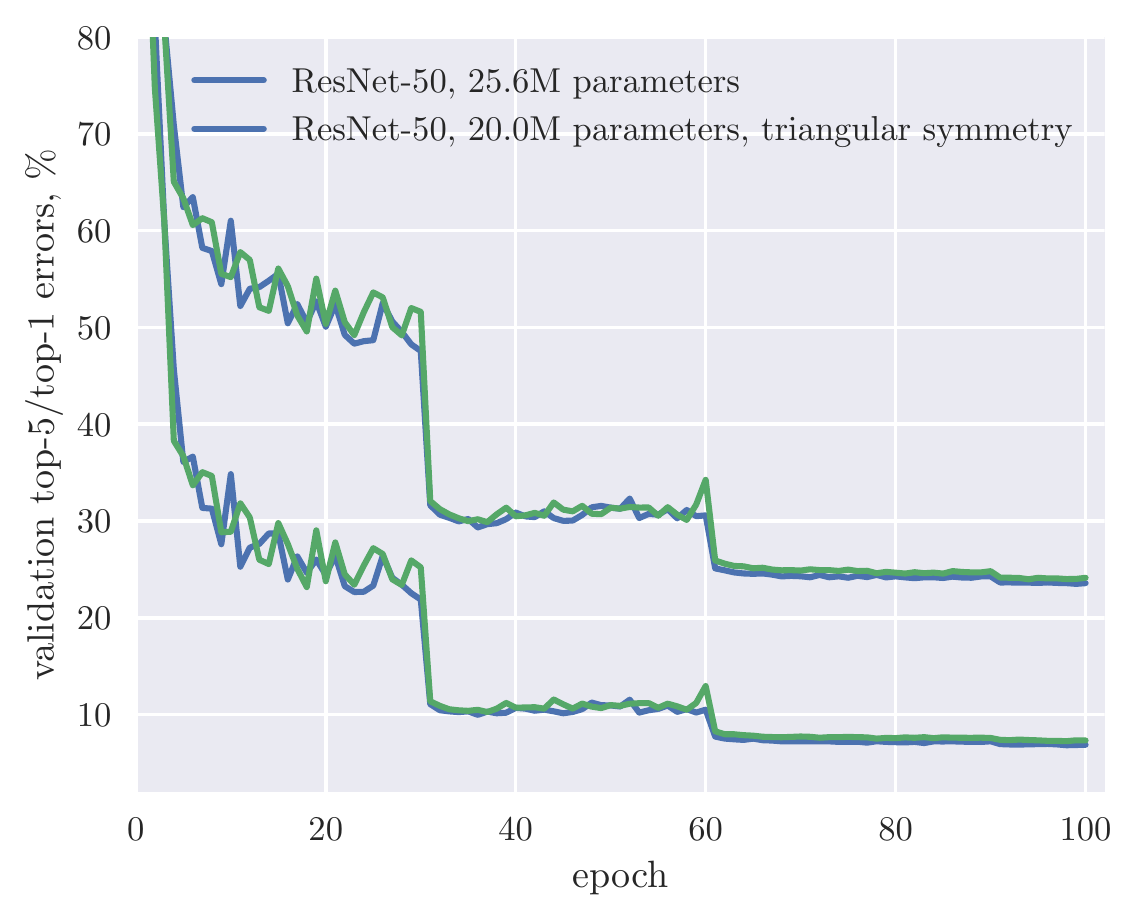}
  }
  \caption{Image classification results for ResNet with symmetric filters.}
\end{figure*}

\subsubsection{Other architectures}
In this section we show that triangular symmetry works as well on other architectures and larger networks. We pick a residual network variant, ResNeXt, which reduces the number of parameters and computational complexity in ResNet by using grouped \(3\times3\) convolution in bottleneck block. For large networks we use WRN-28-10, both basic and bottleneck variants, and a simple feedforward VGG. We also apply symmetry to MobileNets, a popular architecture for mobile devices. Even though being feedforward, it compares favorably to smaller residual networks such as ResNet-18 and ResNet-34, with significant reduction in parameters needed to achieve the same accuracy.

\begin{table*}[t]
  \centering\small
  \begin{tabular}{lcccc}
    \toprule
    network & symmetry & \#params & CIFAR-10 & CIFAR-100 \\
    \midrule
    ResNeXt-16-2-4  &       & 0.57M & 6.93 & 28.3 \\
    ResNeXt-16-2-4  & \chk  & 0.53M & 7.18 & 28.86 \\
    MobileNet       &       & 3.2M  & 7.6  & 31.05 \\
    MobileNet       & \chk  & 2.0M  & 7.91 & 31.48 \\
    VGG             &       & 20M   & 6.11 & 25.75 \\
    VGG             & \chk  & 10.8M & 6.19 & 26.8 \\
    WRN-28-10-basic &       & 36.5M & 3.99 & 18.7 \\
    WRN-28-10-basic & \chk  & 26.8M & 3.97 & 19.1 \\
    WRN-28-10-bottleneck &      & 39.8M & 3.96 & 18.94 \\
    WRN-28-10-bottleneck & \chk & 30.2M & 3.77 & 18.79 \\
    \bottomrule
  \end{tabular}
  \caption{Triangular symmetry applied to various architectures on CIFAR. Triangular channel-wise symmetry is imposed on every second convolution in residual blocks in WRN and ResNeXt, and in all square convolutions in MobileNet and VGG. Median test accuracy of 5 runs is reported.}
  \label{table:cifar_architectures}
\end{table*}

Results are presented in Table~\ref{table:cifar_architectures}. We put triangular symmetry constraint on the second layer in each residual block of WRN-basic-28-10, and on \(3\times3\) convolutional layers in WRN-bottleneck-28-10. In both cases, there are no drops in accuracy, as expected in overparameterized networks which easily achieve 100\% training accuracy. Triangular parameterization works well even with ResNeXt, which has much less parameters in \(3\times3\) layers. That is also
the case for VGG, which, in contrast to others, does not have \(1\times1\) or depth-wise convolutions between layers with symmetry. We believe the fact that there is no big reduction in accuracy is due to the existence of redundant parameters. In MobileNet, we parameterize all square \(1\times\) layers, and observe relatively small drop in accuracy.

Overall, we observe that overparameterized networks can easily benefit from symmetry parameterizations in terms of the number of parameters and potential speedup from more efficient implementation. It is surprising to see that triangular parameterization works well even with ResNeXt and MobileNet, which have already been designed to enjoy weight sharing. Among all these experiments, VGG with triangular parameterization reduces almost half of the parameters (i.e. $9.2$ million), yet the accuracy still remains almost the same.

\subsubsection{Importance of batch normalization}
One might notice that batch normalization~\cite{batch_norm} could potentially be a symmetry-breaking component when combined with symmetric convolutional or linear layers, as it can be viewed as an affine transform of each feature plane, or a diagonal fully connected layer. We, however, successfully apply symmetric parameterization to networks without batch normalization.%

To test the influence of batch normalization on symmetric parameterization, we trained Network-In-Network on CIFAR with and without batch normalization, convergence plots are presented on Fig.~\ref{fig:batch_norm}. As can be seen, the difference in accuracies is very similar in both cases. If batch normalization played a significant role in improving symmetric parameterization, Network-In-Network without batch normalization (left) and triangular symmetry would have a higher accuracy drop. We use learning rate of 0.1 to train the networks with batch normalization, and of 0.01 without. Also, for triangular parameterization without batch normalization we reduce learning rate on upper triangular part by 2 to compensate for gradient magnitude increase due to sharing.

\begin{figure}[ht]
  \includegraphics[width=0.5\textwidth]{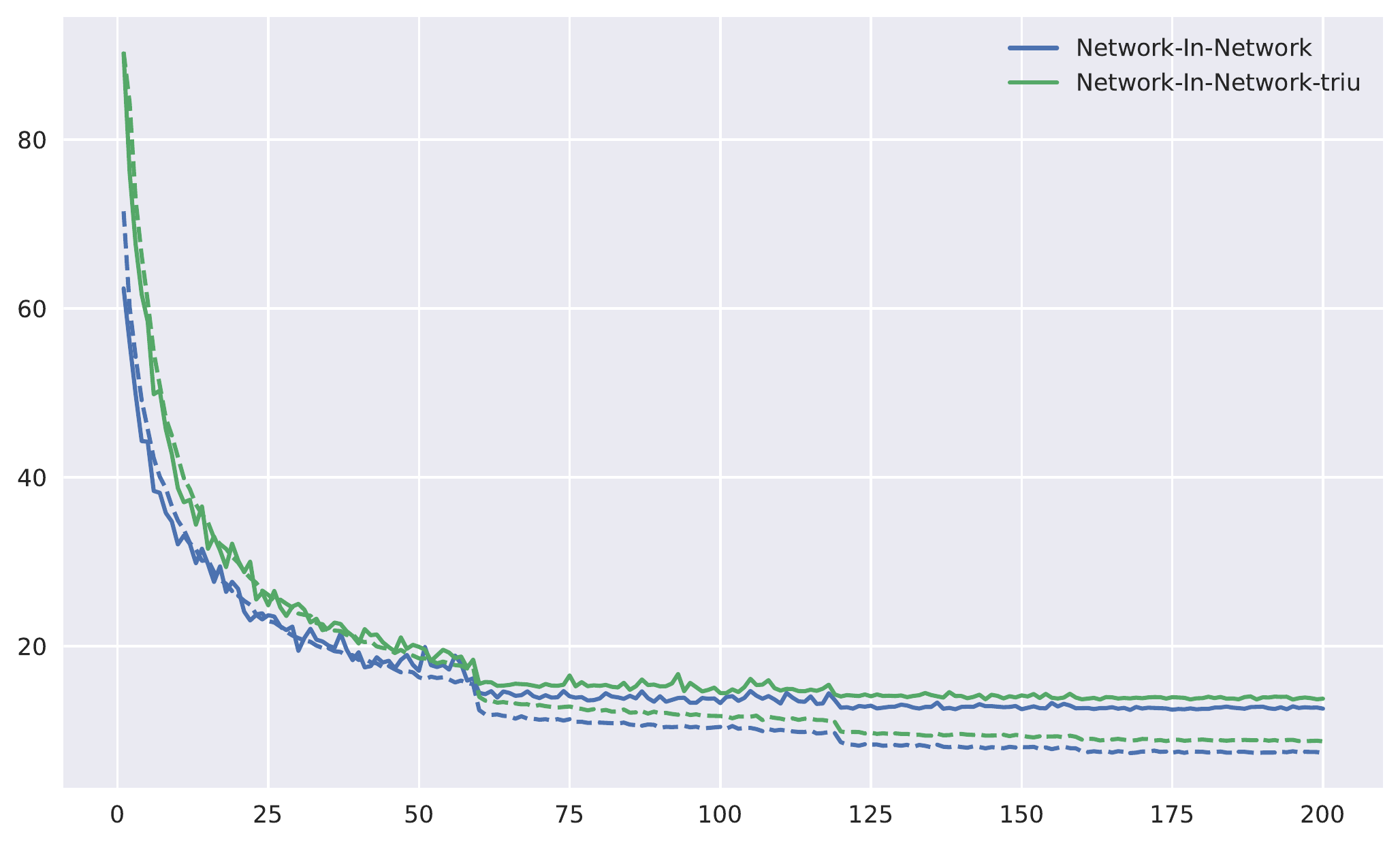}
  \includegraphics[width=0.5\textwidth]{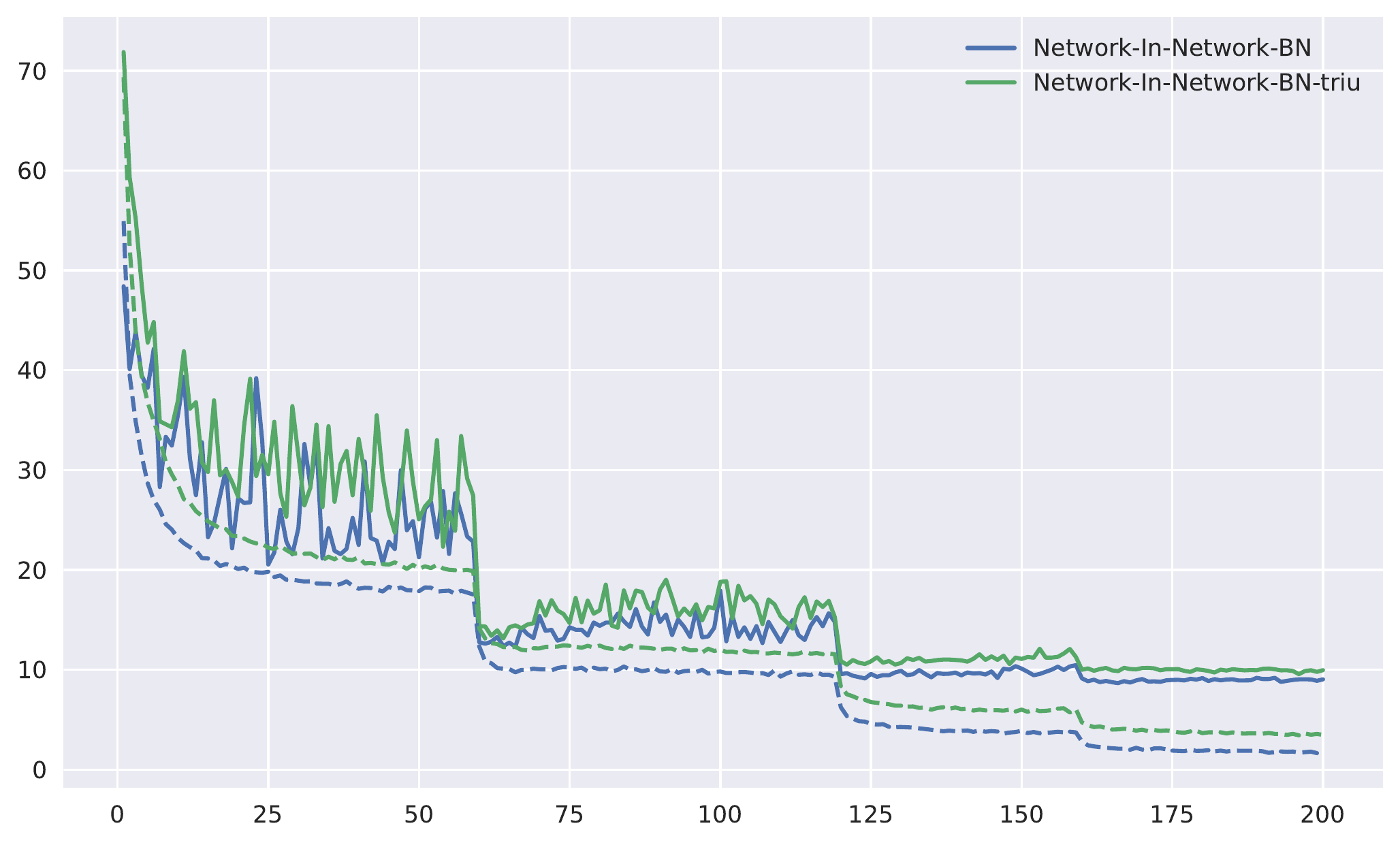}
  \caption{Combining symmetric parameterization without batch normalization (left) and with (right), Network-In-Network. Training accuracy in shown by dashed lines, validation - solid. Accuracy drop is similar in both cases.}
  \label{fig:batch_norm}
\end{figure}

\subsubsection{Combining symmetric parameterization with ShaResNet}
\label{sec:ShaResNet}

We discussed the possibility of combining symmetry parameterizations with other parameter reduction methods in Section~\ref{methods:connection}. Here, we conduct an experiment that combines 2-way channelwise symmetry with ShaResNet \cite{Boulch}, which is a weight sharing method that all \textit{conv1} of residual blocks (bottleneck block) within a group (i.e. layers between two dimensionality reduction convolutions) share the same weights. The results are shown in Table~\ref{table:sharesnet}.
It can be seen that the combination further reduces the number of parameters, while the testing performance is only slightly affected. In particular, with the bottleneck architecture, triangular symmetry pluses ShaResNet reduce about 33\% of parameters and the performance drop is less than 1\%.

\begin{table*}[ht]
  \caption{Combining ShaResNet \cite{Boulch} with 2-way channelwise symmetry. Test errors (mean/std/median over 5 runs) are compared for different symmetry parameterizations on CIFAR-10 using WRN-16-1-bottleneck with symmetric \textit{conv1}.}
  \centering
  \begin{tabular}{lcccccc}
    \toprule
    network & share & symmetry & \#params & mean & std & median \\
    \midrule
    \multirow{5}{*}{WRN-16-1-bottleneck}
    &      & none         & 0.219M    & 8.40 & 0.24 & 8.49 \\
    & \chk & none         & 0.171M    & 8.64 & 0.15 & 8.63 \\
    &      & triangular   & 0.172M    & 8.76 & 0.19 & 8.84 \\
    & \chk & triangular   & 0.147M    & 9.49 & 0.34 & 9.36 \\
    & \chk & average      & 0.171M    & 9.35 & 0.23 & 9.25 \\
    \bottomrule
  \end{tabular}
  \label{table:sharesnet}
\end{table*}

\subsubsection{N-way symmetries}
As discussed in Section~\ref{sec:hard_constr}, it is possible to push forward the triangular parameterization to a more general N-way triangular parameterization. In addition to triangulizing and blocking, we also consider a chunking implementation (a naive N-way weight sharing: given $V$ as a $M \times \frac{M}{N}$ matrix. We define by $\text{tile}_{N\times}(V)$ a transforming function to tile/copy $V$ $N$ times to construct $\hat{W} = \text{tile}_{N\times}(V)$) as an example to show that carefully designed N-way symmetries yield better performance than a naive N-way weight sharing.
Recall that standard triangular parameterization is equivalent to 2-way triangulizing. In this section, we test our proposals including chunking, blocking and triangulizing to achieve N-way channelwise symmetry. For chunking, we examine the cases of $N = 2, 4, \ldots, 64$. If the number of channels in a convolutional layer is less than N, we simply set N equal to the number of channels. For blocking and triangulizing, it is non-intuitive to come up
with N-way symmetry for $N > 4$, so we only test the cases of $N = 4, 16$ and $N = 2, 4, 8$ respectively.

In our experiments, as shown in Fig.~\ref{fig:chunk,block,triang}, chunking causes the steepest linear decrease in validation accuracy. Triangular parameterizations work fine even in the case of 8-way symmetry, which reduces almost $\frac{3}{8}$ percentage of parameters from 2-way symmetry while only suffer about 1\% decrease in accuracy.

\begin{figure}[ht]
  \centering
    \includegraphics[width=0.7\columnwidth]{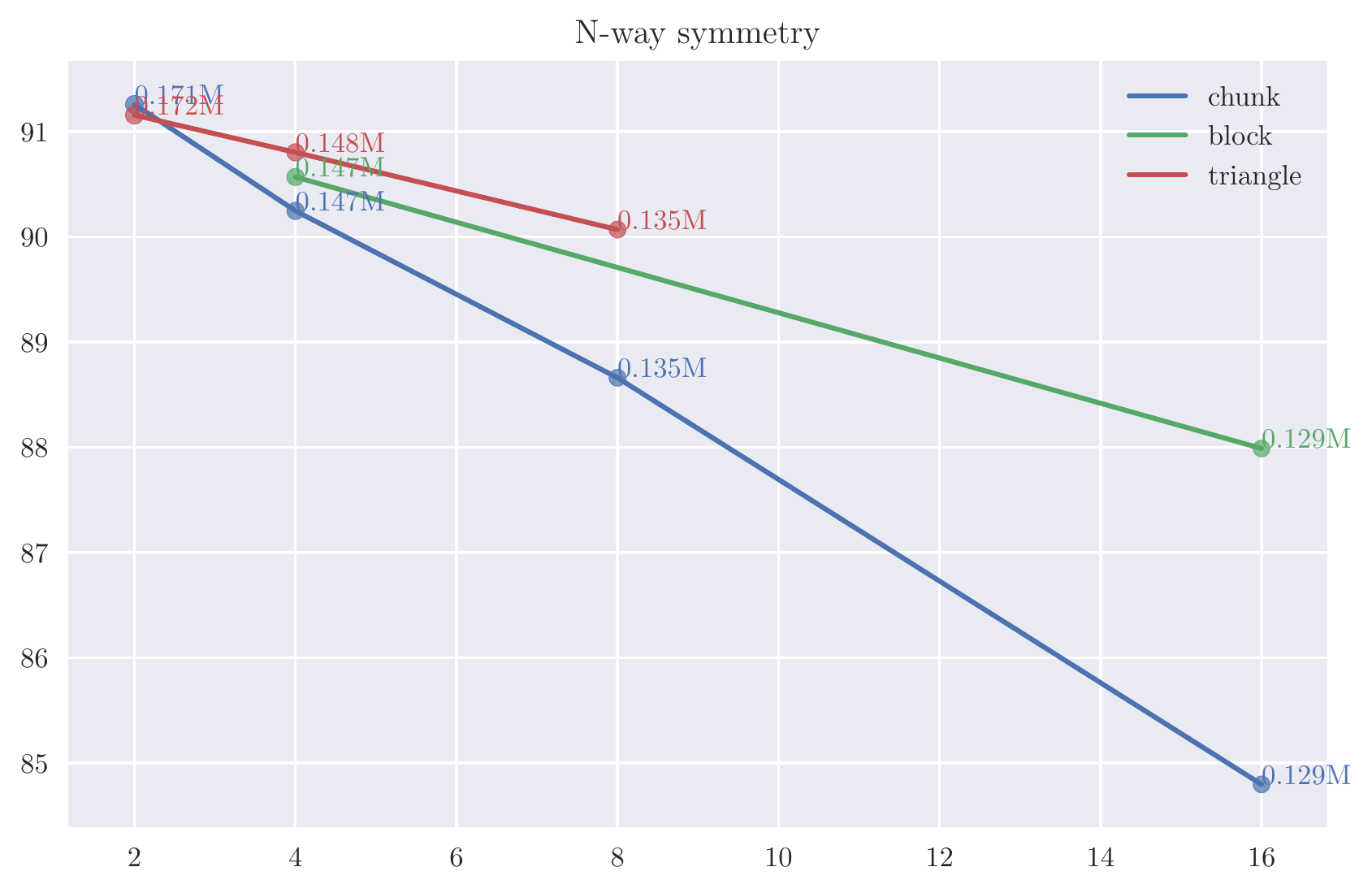}
    \caption{N-way sharing (chunking) v.s. N-way symmetries (blocking, triangulizing) on CIFAR-10 with WRN-16-1-bottleneck. x-axis represents N. y-axis represents the accuracy.}
    \label{fig:chunk,block,triang}
\end{figure}

\subsection{ImageNet experiments}

In this section we present ImageNet results for networks with triangular symmetry and without, to check if our conclusions from CIFAR transfer to larger dataset.

We start with results for relatively small networks, MobileNet and ResNet-18. In MobileNet we impose triangular symmetry on all square \(1\times1\) convolutional layers. ResNet-18 has basic block architecture and doesn't have enough parameters to fit into training data well, so we impose symmetry on every second \(3\times3\) convolutional layer in block. Both MobileNet and ResNet-18 are relatively shallow networks and have 28\% less parameters with triangular symmetry, so, as expected, the drop in accuracy is significant, see Table~\ref{table:imagenet}. Still, constrained MobileNet has only 3M parameters and achieves almost the same accuracy with ResNet-18.

\begin{table*}
  \centering\small
  \begin{tabular}{lcccc}
    \toprule
    network & symmetry & \#params & top-1 & top-5 \\
    \midrule
    MobileNet &      & 4.2M  & 28.18 & 9.8 \\
    MobileNet & \chk & 3.0M  & 30.57 & 11.6 \\
    ResNet-18 &      & 11.8M & 30.54 & 10.93 \\
    ResNet-18 & \chk & 8.6M  & 31.44 & 11.55 \\
    ResNet-50 &      & 25.6M & 23.50 & 6.83 \\
    ResNet-50 & \chk & 20.0M & 23.98 & 7.25 \\
    ResNet-101 &     & 44.7M & 22.14 & 6.09 \\
    ResNet-101 &\chk & 34.0M & 22.36 & 6.35 \\
    \bottomrule
  \end{tabular}
  \caption{ImageNet results for networks with triangular symmetry parameterization. Smaller networks such as MobileNet and ResNet-18 have more significant drop in accuracy than larger ResNet-50 and ResNet-101. The latter have 23\% less parameters than non-symmetric counterparts and have drops in accuracy of 0.5\% and 0.2\% correspondingly}
  \label{table:imagenet}
\end{table*}

As for large networks, we trained ResNet-50 and ResNet-101 with bottleneck block architecture and triangular symmetry on all \(3\times3\) layers. As on CIFAR, drop in accuracy is much smaller for these: a reduction of 23\% parameters (which would correspond to approximately ResNet-40 for ResNet-50 and ResNet-77 for ResNet-101) causes only 0.5\% accuracy drop compared to unconstrained ResNet-50, and even smaller for ResNet-101, which is about 0.2\%. We show convergence curves for both ResNet-50 and its symmetric variant on Fig.~\ref{fig:imagenet_resnet50}.%

All networks were trained in the same conditions and with the same hyper-parameters. We used large mini-batch training approach as proposed in~\cite{large_minibatch} on 8 GeForce 1080Ti GPUs, scaling learning rate proportionally to mini-batch size. Also, we do not regularize batch normalization and depth-wise convolution parameters in MobileNet. Surprisingly, our MobileNet and ResNet baselines outperform the original networks proposed in \cite{mobilenets} and \cite{he2015deep}. We plan to make our code and networks available for download online.

\subsection{Language modeling}
\label{sec:language}

We test symmetry on a common language modeling Penn-Tree-Bank~\citep{Marcus1993BuildingAL} dataset. Our experimental setup reflects that of \citep{zarembalstm}\footnote{\scriptsize\url{https://github.com/pytorch/examples/tree/master/word_language_model}}. We use a network with 2-layer LSTM and dropout, and test out a medium size configuration with 650 neurons (dropout 0.5), and a larger model with 1500 neurons (dropout 0.65).
We try to symmetrize weights corresponding to input and hidden and gates separately and altogether, as well for each gate separately. Surprisingly, we find that adding symmetry on all hidden gates does not hurt, and even obtain slightly lower perplexity for both medium and large models. The results are presented in Table~\ref{table:penntree}. We use the average parameterization $\frac{1}{2}(\mW + \mW^\top)$ as we notice it gives slightly better results. Also, there is no
$\normltwo$-regularization applied during training, so the original $\mW$ weights converge to be almost symmetric. There is also a large variation in final validation perplexity, so we train each network 5 times with different random seed and report mean$\pm$std results for all models.

\begin{table}
  \centering
  \begin{tabular}{lcccccc}
    \toprule
    \multirow{2}{*}{Model}  & \multirow{2}{*}{symmetry} & \multirow{2}{*}{\#parameters} & \multicolumn{2}{c}{\texttt{Penn-Tree-Bank}} & \multicolumn{2}{c}{\texttt{wikitext-2}} \\
     &  &  & validation & test & validation & test \\
    \midrule
    LSTM-650 &              & 6.8M &  $85.38\pm0.29$ & $81.49\pm0.04$ & $99.59\pm0.17$ & $94.26\pm0.21$  \\
    LSTM-565 &              & 5.1M &  $85.48\pm0.45$ & $81.47\pm0.28$ & $100.60\pm0.28$ & $94.99\pm0.24$ \\
    LSTM-650 & \chk         & 5.1M &  $83.73\pm0.42$ & $79.73\pm0.23$ & $100.81\pm0.45$ & $95.43\pm0.37$ \\
    \midrule
    LSTM-1500 &             & 36M &  $81.90\pm0.54$ & $77.95\pm0.41$ &  $95.59\pm0.34$ & $90.92\pm0.20$ \\
    LSTM-1300 &             & 27M &  $80.71\pm0.14$ & $77.42\pm0.18$ &  $96.08\pm0.26$ & $90.77\pm0.21$ \\
    LSTM-1500 & \chk        & 27M &  $79.66\pm0.41$ & $75.69\pm0.32$ &  $97.13\pm0.77$ & $91.97\pm0.85$ \\
    \bottomrule
  \end{tabular}
  \caption{Language modeling perplexity (lower is better) on Penn-Tree-Bank and wikitext-2 datasets. Only hidden gate weights are symmetrized with triangular parameterization. We count only parameters in RNN, skipping encoder and decoder. Mean$\pm$std results over 5 runs are reported.}
  \label{table:penntree}
\end{table}

For a fair comparison we also trained thinner networks with 565 and 1300 hidden neurons for medium and large networks correspondingly, so that the number of parameters is approximately equal to those of hidden-symmetrized networks. Symmetrized versions of large networks compare favourably to these networks.

We also include experimental results on a larger wikitext-2~\citep{Merity2016PointerSM} dataset in Table~\ref{table:penntree}, which is about 2 times larger than Penn-Tree-Bank, as well experiments with symmetry on each gate separately.
We apply averaging symmetry to hidden gates of LSTM and compare to thinner networks with comparable number of
parameters. In this case symmetry slightly hurts perplexity of medium model, and more significantly large model. This might be due to different dropout regularization in the networks (we use the same dropout rates as in Penn-Tree-Bank\footnote{Suggested by \scriptsize\url{https://github.com/pytorch/examples/tree/master/word_language_model}}, which may not be the best choices for wikitext-2).%

\section{Conclusions}

In this paper, we have shown that, quite surprisingly, deep neural network weights can be successfully parameterized to be symmetric without suffering a significant loss in accuracy. The proposed symmetry parameterizations could lead to potentially significant improvements for a wide range of mobile applications in terms of computational efficiency (dedicated routines for symmetric convolutions and matrix multiplication can be applied) and storage efficiency (memory requirements for storing network weights are dramatically reduced).

For future work, it would be interesting to compare with other structural weight matrices (e.g.,~\cite{zhao2017theoretical}) with computational benefits and understand what kinds of inductive bias are implied. It should also be noted that our universal approximation analysis holds only for approximating single univariate continuous functions. It remains an open question what theoretical guarantees a symmetric deep neural network can provide for more general multivariate functions.

\section*{Acknowledgements}

We thank Alexander~Khanin and Soumik Sinharoy for providing us dedicated deep learning resources, without which this work would not be possible. S. Zagoruyko was also supported by the DGA RAPID project DRAAF.

\bibliography{bibliography}
\bibliographystyle{apalike}

\end{document}